\documentclass[submission,copyright,creativecommons]{eptcs}

\usepackage{graphicx}
\usepackage{tikz}
\usepackage{amsthm}
\usepackage{amsfonts}
\usepackage[british]{babel}
\usepackage[utf8]{inputenc}

\usetikzlibrary{shapes.geometric, arrows}

\tikzstyle{startstop} = [rectangle, rounded corners, minimum
width=3cm, minimum height=1cm,text centered, draw=black, fill=red!30] 

\tikzstyle{io} = [trapezium, trapezium left angle=70, trapezium right
angle=110, minimum width=3cm, minimum height=1cm, text centered,
draw=black, fill=blue!30] 

\tikzstyle{process} = [rectangle, minimum width=3cm, minimum
height=1cm, text centered, draw=black, fill=orange!30]

\tikzstyle{decision} = [diamond, minimum width=3cm, minimum
height=1cm, text centered, draw=black, fill=green!30] 

\tikzstyle{arrow} = [thick,->,>=stealth]

\newtheorem{definition}{Definition}
\newtheorem{theorem}{Theorem}
\newtheorem{lemma}{Lemma}

\providecommand{\event}{Proceedings of ADG 2023} 

\title{Considerations on Approaches and Metrics in Automated Theorem
  Generation/Finding in Geometry}

\author{Pedro Quaresma%
  \thanks{Partially supported by FCT -- Foundation for Science and
    Technology, I.P., within the scope of the project CISUC --
    UID/CEC/00326/2020 and by European Social Fund, through the
    Regional Operational Program Centro 2020.}%
  \institute{
    CISUC / Department of Mathematics, \\
    University of Coimbra, Portugal}\email{pedro@mat.uc.pt}\and
  Pierluigi Graziani%
  \thanks{Partially supported by Italian Ministry of Education,
    University and Research through the PRIN 2017 project ``The
    Manifest Image and the Scientific Image'' prot. 2017ZNWW7F\_004.}%
  \institute{Department of Pure and Applied Sciences, University of Urbino, Italy}
  \email{pierluigi.graziani@uniurb.it}
  \and
  Stefano M. Nicoletti%
  \thanks{Funded by ERC
    Consolidator Grant 864075 (\emph{CAESAR}).}%
  \institute{Formal Methods and Tools, University of Twente, Enschede, the Netherlands}
  \email{s.m.nicoletti@utwente.nl}
  }

\def\authorrunning{P. Quaresma, P. Graziani \& S.M. Nicoletti} 

\def\titlerunning{Considerations on Approaches and Metrics in ATG/F in Geometry}

\begin{document}

\maketitle


\begin{abstract}
  

  The pursue of what are properties that can be identified to permit
  an automated reasoning program to generate and find new and
  interesting theorems is an interesting research goal (pun
  intended). The automatic discovery of new theorems is a goal in
  itself, and it has been addressed in specific areas, with different
  methods. The separation of the ``weeds'', uninteresting, trivial
  facts, from the ``wheat'', new and interesting facts, is much
  harder, but is also being addressed by different authors using
  different approaches.  In this paper we will focus on geometry.
  We present and discuss different approaches for the automatic
  discovery of geometric theorems (and properties), and different
  metrics to find the interesting theorems among all those that were
  generated. After this description we will introduce the first result
  of this article: an undecidability result proving that having an
  algorithmic procedure that decides for every possible Turing Machine
  that produces theorems, whether it is able to produce also
  interesting theorems, is an undecidable problem. Consequently, we
  will argue that judging whether a theorem prover is able to produce
  interesting theorems remains a non deterministic task, at best a
  task to be addressed by program based in an algorithm guided by
  heuristics criteria.  Therefore, as a human, to satisfy this task
  two things are necessary: an expert survey that sheds light on what
  a theorem prover/finder of interesting geometric theorems is,
  and---to enable this analysis---other surveys that clarify metrics
  and approaches related to the interestingness of geometric
  theorems. In the conclusion of this article we will introduce the
  structure of two of these surveys ---the second result of this
  article--- and we will discuss some future work.

\end{abstract}

\section{Introduction}
\label{sec:introduction}

In \emph{Automated Reasoning: 33 Basic Research Problems}, Larry Wos,
wrote about the problems that computer programs that reason
face. Problem 31 is still open and object of active
research~\cite{Wos1988,Wos1993}:

\begin{quote}
  \emph{Wos' Problem 31}---What properties can be identified to permit
  an automated reasoning program to find new and interesting theorems,
  as opposed to proving conjectured theorems?
\end{quote}

Two problems in a single sentence: \emph{new} and \emph{interesting}
theorems. The automatic discovery of new theorems is a goal in
itself, it has been addressed in specific areas, with different
methods. The separation of the ``weeds'', uninteresting, trivial
facts, from the ``wheat'', new and interesting facts, is much harder,
but is being addressed also, by different authors using different
approaches.

Paraphrasing, again, Wos, ``since a reasoning program can be
instructed to draw some (possible large) set of conclusions'' what
should be the ``criteria that permit the program to select from those
the ones (if any) that correspond to interesting results.''

Different fields have come across the finding of new and interesting
theorems' questions.



Regarding the novelty side: there are different views of approaching
new mathematical results. One of those approaches is the systematic
exploration of a given broad area of mathematical knowledge,
generating, by different means, new theorems and expecting to find
interesting ones among those generated (that will be analysed in
section~\ref{sec:InterestingTheorems})~\cite{Cheng1995,Cheng2000,Colton2000,Gao2015,Gao2019,Kovacs2021c,Kovacs2022a,Puzis2006}.
Another approach is given by the pursue of mathematical discovery in
specific areas, e.g. \emph{Computing Locus
  Equations}~\cite{Abanades2014,Botana2007}, \emph{Automated Discovery
  of Angle Theorems}~\cite{Todd2021}, \emph{Automated Discovery of
  Geometric Theorems Based on Vector Equations}~\cite{Peng2021},
\emph{Automated Generation of Geometric Theorems from Images of
  Diagrams}~\cite{Chen2014}, \emph{Automatic Discovery of Theorems in
  Elementary Geometry}~\cite{Recio1999}. These approaches do not aim
to address the problem of automated theorem finding in itself but, for
example, to find complementary hypotheses for a given geometric
statements to become true~\cite{Recio1999} i.e. automatic discovery
for specific areas.\footnote{We left aside the notion of discovery in
education, given that, in that area, the goal is the student's
discovery of ``new'' (for them) theorems, giving the student the
possibility of freely making conjectures and having an
interactive/automatic deduction support in the exploration of those
``new'' theorems~\cite{Botana2002b,Hanna2019,icmi19v1,icmi19v2}.}




Regarding the interestingness side we are aware that relevant
literature can be found in different areas. For example in automated
theorem proving~\cite{Colton2000,Gao2015,Gao2019,Puzis2006} and in
sociological studies on the concept of
proving~\cite{de1979social,mackenzie1995automation,mackenzie1999slaying},
in cognitive and educational science studies on the concept of
proving~\cite{aldon2010,Bundy2005,d2016formula,hemmi2017misconceptions,polya2004solve,stylianides2018advances}
and in semiotics and epistemology of
mathematics~\cite{Arzarello2010,aschbacher2005highly,Avigad2006,Barendregt2005,burge1998computer}.

Despite the cited studies, the Wos' problem is still on the table. On the contrary, a new result of undecidability can be added to the problem, i.e. having an algorithmic procedure that decides for every possible Turing Machine that produces theorems, whether it is able to produce also interesting theorems, is an undecidable problem.
Consequently, we can argue that judging whether a theorem prover is able to produce interesting theorems remains a non deterministic task, at best a task to be addressed by program based in an algorithm guided by heuristics criteria.  Therefore, as a human, to satisfy this task we need expert survey that sheds light on what a theorem prover/finder of interesting geometric theorems is, and---to enable this analysis---other surveys that clarify metrics and approaches related to the interestingness of geometric theorems.






\textbf{Structure of the paper.} In section~\ref{sec:TheoremDiscovery}
the issue of Automated Theorem Generation (ATG) is discussed.  In
section~\ref{sec:thedeductiveapproach} we discuss the deductive
approach in ATG.  In section~\ref{sec:InterestingTheorems} the issue
of Automated Theorem Finding (ATF) is analysed. In
section~\ref{sec:undecidabilityResult} we present an undecidability
result concerning the problem of finding interesting theorems and its conceptual consequences. In section~\ref{sec:surveys} we will introduce the structure of two surveys to empirically explore the interestingness of theorems in geometry and its potential application in theorem proving/finding (a third survey). Finally, we will discuss some future work.


\section{Automated Theorem Generation}
\label{sec:TheoremDiscovery}

Automated theorem generation, independently of being interesting, or
not, can be addressed in several ways~\cite{Puzis2006}.

\paragraph{The Inductive approach,}
\label{par:theinductiveapproach}
is a natural approach. Conclusions are drawn by going from the
specific to the general. Exploring a given domain, seeking for
properties that emerge from a set of particular cases and making a
conjecture about the general case.

Dynamic Geometry Software (DGS) can be seen as software environments
to inductively explore new knowledge. Making a geometric construction,
constrained by a given set of geometric properties, and then moving
the free point around will show all the fix-points, conjecturing if
those new fixed relations between objects are true in all cases, or
not. For example the Pappus' Theorem, in this case, a well-known
theorem: are the intersection points (see Figure~\ref{fig:pappus})
$G$, $H$ and $I$, collinear? By moving, in the DGS, the free-points it
seems that they are, it remains to prove it.

\begin{figure}[htbp!]
  \begin{center}
    \includegraphics[width=0.35\textwidth]{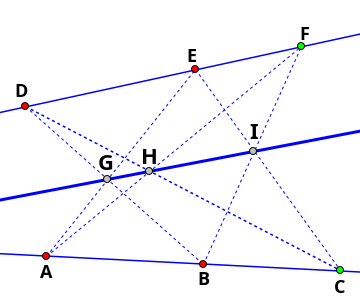}
  \end{center}
  \caption{Pappus' Theorem}
  \label{fig:pappus}
\end{figure}

The inductive approach has the advantage of being stimulated by
observations in the domain. but has the disadvantage that induction is
unsound. A famous example of such unsound inductive approach can be
seen in the Euclid Parallel lines Postulate, that nevertheless was
very fruitful, giving raise to different geometries.


\paragraph{The Generative approach,}
\label{par:thegenerativeapproach}
i.e. the generation of conjectures, testing them for theorem-hood. The
simplest form of generation is syntactic, in which conjectures are
created by mechanical manipulation of symbols,
e.g.~\cite{Plaisted1994}. The \emph{MCS\/} program generates
conjectures syntactically and filters them against models of the
domain~\cite{Zhang1999}. A stronger semantically based approach is
taken by the \emph{HR\/} program, which generates conjectures based on
examples of concepts in the domain~\cite{Colton2002}. A theory
exploration system called \emph{QuickSpec}, works by interleaving term
generation with random testing to form candidate
conjectures~\cite{Johansson2021}. In~\cite{Johansson2021} the
conjecture generation approaches are classified into three categories:
heuristic rule-based systems, term generation-and-testing and neural
network-based systems.  The \emph{RoughSpec} system adds to
\emph{QuickSpec} the notion of \emph{shapes} of theorems, specifying
the shapes of conjectures the user is interested in, and thus
limiting the search~\cite{Einarsdottir2020}.

Like induction, generation is unsound. However, if the rules by which
the generation is performed are sufficiently conservative then this
approach may generate a higher fraction of theorems than the inductive
approach.

\paragraph{The Manipulative Approach,}
\label{par:themanipulativeapproach}
conjectures are generated from existing theorems. An existing theorem
is modified by operations such as generalisation, specialisation,
combination, etc. This approach is used in abstraction mapping, which
converts a theorem to a simpler theorem, and uses a solution to the
simpler theorem to help find a solution to the original
theorem~\cite{Plaisted1980}. Manipulation of ATP theorems has also
been used to produce new theorems for testing the robustness of ATP
systems' performances~\cite{Voronkov2000}.

An advantage of the manipulative approach is that, if the
manipulations are satisfiability preserving, then theorems, rather
than conjectures, are produced from existing theorems. However, the
conjectures produced by the manipulative approach are typically
artificial in nature, and thus uninteresting.

\paragraph{The Deductive Approach,}
\label{par:thedeductiveapproach}
consequences are generated by application of sound inference rules to
the axioms and previously generated logical consequences. This can be
done by an appropriately configured saturation-based ATP system.

The advantage of this approach is that only logical consequences are
ever generated. The challenge of this approach is to avoid the many
uninteresting logical consequences that can be generated.

\section{The Deductive Approach}
\label{sec:thedeductiveapproach}

Some systems addresses, explicitly, the generation of new geometric
results using different approaches. In the following some of these
approaches are described.

\subsection{Strong Relevant Logic-based Forward Deduction Approach}
\label{sec:StrongRelevantLogicBasedForwardDeductionApproach}


In~\cite{Gao2014a} the authors argue for the fundamental difference
between the Automated Theorem Proving (ATP) and the Automated Theorem
Finding (ATF). ATP is the process of finding a justification for an
explicitly specified statement from given premises which are already
known facts or previously assumed hypotheses. ATF is the process to
find out or bring to light that which was previously unknown. Where
ATP is all about known (old) facts, ATF is about previously unknown
conclusions from given premises.  Jingde Cheng~\cite{Cheng2000} claims
that classical mathematical logic, its various classical conservative
extensions, and traditional (weak) relevant logics cannot
satisfactorily underlie epistemic processes in scientific discovery,
presenting an approach based on strong relevant logic. Hongbiao Gao et
al. have followed this approach applying it for several domains such
as NBG set theory, Tarski's Geometry and Peano's
Arithmetic~\cite{Gao2017,Gao2014a,Gao2018,Gao2019}

\subsection{Rule Based Systems}
\label{sec:DeductiveDatabases}

The rule-based automated deduction system are often used when the
proof itself is an object of interest (and not only the end result),
given that the proofs are developed from the hypothesis and sets of
axioms, to the conclusion by application of the inference rules, the
proofs are ``readable''.

Example of such approaches can be seen in systems like
\emph{QED-Tutrix}~\cite{Font2021,Font2018} and
\emph{JGEx}~\cite{Ye2011}, both for geometry. In the tutorial system
\emph{QED-Tutrix}, the rule based automated theorem prover goal is to
find the many possible branches of the proof tree, in order to be able
to help the student approaching the proof of a geometric conjecture. In
the \emph{JGEx} system we can have the proof in a ``readable'' and
``visual'' renderings and also the set of all properties that can be
deduced from the construction.

One of the ATP built-in in \emph{JGEx} is an implementation of the
geometry deductive database method~\cite{Chou2000,Ye2011}. Using a
breadth-first forward chaining a fix-point for the conjecture at
hand is reached. For that geometric construction and the rules of the
method, the fix-point gives us all the properties that can be
deduced, some already known facts, but also new facts (not necessary
interesting ones).

The geometry deductive database method proceeds by using a simple
algorithm where, starting from the geometric construction $D_0$, the
rules, $R$, are applied over and over till a fix-point, $D_k$ is
reached:
\begin{equation}
  \framebox{$D_0$}\quad \stackrel{R}{\subset}\quad
  \framebox{$D_1$}\quad \stackrel{R}{\subset}\quad
  \quad\cdots\quad \stackrel{R}{\subset}\quad
  \framebox{$D_k$}\quad (\mbox{fix-point})
  \label{eq:gddmFixedPoint}
\end{equation}

In figure~\ref{fig:jgexFixedPoint} an example, using \emph{JGEx}, is
shown. On the right, the geometric construction, on the left, the
fix-point, with all the facts that were found for that construction.

\begin{figure}[htbp!]
  \begin{center}
    \includegraphics[width=0.7\textwidth]{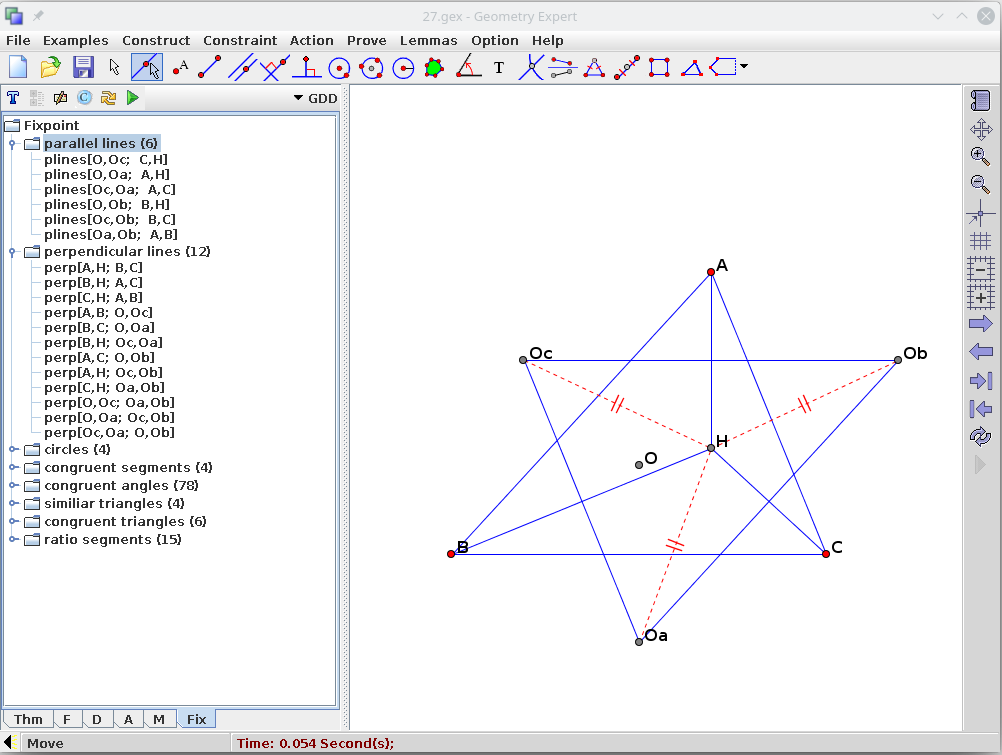}
  \end{center}
  \caption{Fix-point in \emph{JGEx}}
  \label{fig:jgexFixedPoint} 
\end{figure}

A new open source implementation of this method,
\emph{OGP-GDDM},\footnote{\url{https://github.com/opengeometryprover/OpenGeometryProver}}
is described in~\cite{Baeta2023}. It will be integrated in the
\emph{Open Geometry Prover Community Project}
(OGPCP)~\cite{Baeta2021}. One of the medium-term goals of the
\emph{OGP-GDDM} project, is to develop a meta-prover, a program
capable to receive different sets of rules and synthesise a specific
ATP for those rules.

\subsection{Algebraic Approaches}
\label{sec:AlgebraicApproaches}

A similar approach is taken in the well-known dynamic geometry system
\emph{GeoGebra}.\footnote{\url{https://www.geogebra.org/}} The
\emph{GeoGebra Discovery}
version\footnote{\url{https://github.com/kovzol/geogebra-discovery}}
has the capability to find, from a user defined geometric
construction, properties about that construction. \emph{GeoGebra Discovery}
reports some facts that were systematically checked from a list of
possible features including identical points, parallel or
perpendicular lines, equal long segments, collinearity or
concyclicity. This is not a deductive method so the generation process
must be externally verified, \emph{GeoGebra Discovery} do that by recurring
to a built-in algebraic automated theorem prover based in the Gr{\"o}bner
bases method~\cite{Kovacs2021c,Kovacs2022a}.

\section{Automated Theorem Finding}
\label{sec:InterestingTheorems}

Apart from our research goal of finding the interesting geometric
theorems among all those that were automatically generated, the pursue
of measures of interestingness has applicability in the interactive
and automated theorem proving area. In that area a common use of
interestingness is to improve the efficiency of the programs,
tailoring the search space, making the search depth limited and
guaranteeing that only comprehensible concepts are
produced~\cite{Colton2000}.

A goal, pursued with different approaches by many researchers, is the
creation of strong AI methods capable of complex research-level
proofs, mathematical discovery, and automated formalisation of today's
vast body of mathematics~\cite{Rabe2021}.  The \emph{MATHsAiD}
(Mechanically Ascertaining Theorems from Hypotheses, Axioms and
Definitions) project aimed to build a tool for automated
theorem-discovery, from a set of user-supplied axioms and
definitions. In the words of its authors, \emph{MATHsAiD 2.0} can
conjecture and prove interesting Theorems in high-level theories,
including Theorems of current mathematical significance, without
generating an unacceptable number of uninteresting
theorems~\cite{McCasland2017}.  The \emph{TacticToe} system, combines
reinforcement-learning with Monte-Carlo proof search on the level of
\emph{HOL4} tactics~\cite{Gauthier2021}. The \emph{ENIGMA-NG} system
uses efficient neural and gradient-boosted inference guidance for the
ATP \emph{E}, improving its efficiency~\cite{Chvalovsk2019}.  This two
systems, one for interactive provers and the other to automatic
provers, are examples of systems that uses discovery and filtering for
improving the efficiency of automated deduction systems.

\subsection{The Deductive Approach  Algorithm}
\label{sec:overallalgorithm}

The different approaches found in the
literature~\cite{Colton2002,Gao2014a,Puzis2006} share, in their general
lines, the same algorithm: for a given logical fragment, select a
initial set of facts and then a cycle of generation/filtering is
applied until some stopping condition is matched (see
Fig.~\ref{fig:NewInterestingTheoremsAlgorithm}).

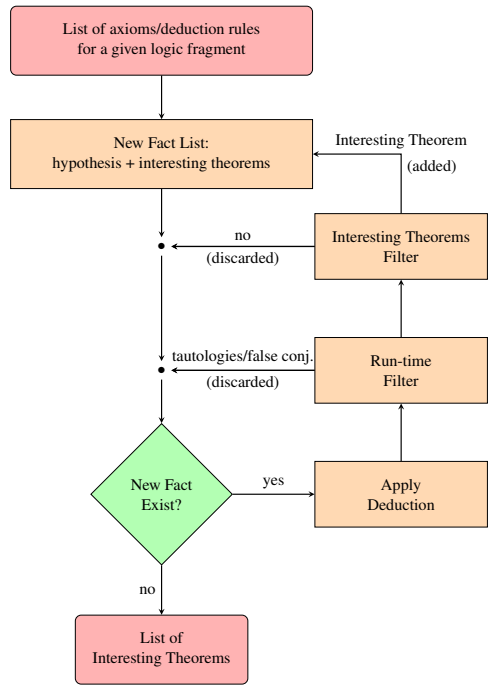
\begin{figure}[htbp!]
  \begin{center}
    \resizebox{0.4\textwidth}{!}{
      \begin{tikzpicture}[node distance=2.7cm]
        \node (start) [startstop] {\parbox{18em}{\begin{center}List of
              axioms/deduction rules\\ for a given logic fragment\end{center}}};
        \node (proc1) [process, below of=start] {\parbox{18em}{\begin{center}New
              Fact List:\\ hypothesis + interesting theorems\end{center}}};
        \node (preDec1) [below of=proc1, yshift=0.5cm] {$\bullet$};
        \node (preDec2) [below of=preDec1, yshift=-0.25cm] {$\bullet$};
        \node (dec1) [decision, below of=preDec2, yshift=-0.25cm]
            {\parbox{5.5em}{\centering{New Fact}\\ \centering{Exist?}}};
        \node (apply) [process, right of=dec1, xshift=3cm]
            {\parbox{10em}{\begin{center}Apply \\ Deduction\end{center}}};
        \node (proc2a) [process, above of=apply, yshift=0.25cm]
            {\parbox{10em}{\begin{center}Run-time\\ Filter\end{center}}};
        \node (proc2b) [process, above of=proc2a, yshift=0.25cm]
           {\parbox{10em}{\begin{center}Interesting Theorems\\ Filter\end{center}}};
        \node (stop) [startstop, below of=dec1, yshift=-1cm]
           {\parbox{10em}{\begin{center}List of\\ Interesting Theorems\end{center}}};
        \draw [arrow] (start) -- (proc1);
        \draw [arrow] (proc1) -- (preDec1);
        \draw [arrow] (dec1) -- node[anchor=south] {yes} (apply);
        \draw [arrow] (preDec1) -- (preDec2);
        \draw [arrow] (preDec2) -- (dec1);
        \draw [arrow] (apply) --  (proc2a);
        \draw [arrow] (proc2a) --  (proc2b);
        \draw [arrow] (dec1) -- node[anchor=east] {no} (stop);
        \draw [arrow] (proc2b) |- node[anchor=south] {Interesting Theorem}
        node[anchor=north west] {(added)} (proc1);
        \draw [arrow] (proc2b) --  node[anchor=north] {(discarded)} (preDec1);
        \draw [arrow] (proc2b) --  node[anchor=south] {no} (preDec1);
        \draw [arrow] (proc2a) --  node[anchor=north] {(discarded)} (preDec2);
        \draw [arrow] (proc2a) --  node[anchor=south]
        {tautologies/false conj.} (preDec2);
      \end{tikzpicture}}
  \end{center}
  \label{fig:NewInterestingTheoremsAlgorithm}
  \caption{New and Interesting Theorems Algorithm}
\end{figure}

\subsection{Filtering Interesting Theorems}
\label{sec:PuzzisClassification}

A first level of filtering (run-time filter) should discard the
obvious tautologies and also conjectures proved false by
empirical evidence. 

The filtering for interesting theorems or for uninteresting
conjectures, two sides of the same coin, is done by application of a
series of filters.  These filters are still to be validated, being of
speculative nature~\cite{Colton2000,Gao2018,Gao2019,Puzis2006}.

\begin{description}
\item[Obviousness:] the number of inferences in its
  derivation. Obviousness estimates the difficulty of proving a
  formula, it can be given by the number of inferences in its
  derivation. 
\item[Weight:] the effort required to read a formula. The weight score
  of a formula is the number of symbols it contains. 
\item[Complexity:] the effort required to understand a formula, the
  number of distinct function and predicate symbols it contains. 
\item[Surprisingness:]  measures new relationships between
  concepts and properties. 
\item[Intensity:] measures how much a formula summarises information
  from the leaf ancestors in its derivation tree.
\item[Adaptivity:] measures how tightly the universally quantified
  variables of a formula are constrained (for formulae in clause
  normal form). 
\item[Focus:] measures the extent to which a formula is making a
  positive or negative statement about the domain of application. 
\item[Usefulness:] measures how much an interesting theorem has
  contributed to proofs of further interesting theorems.
\end{description}

In spite of the relevance of these metrics, it would be appropriate to
construct an expert survey with which we could validate them by
referring to a significant public of experts. We believe this kind of
survey would be relevant not only to face Wos' problem, but also to
better understand how to construct and evaluate software that
generates/finds interesting theorems. Despite having only relevant
metrics and approaches regarding Wos' problem, while not yet having
formal results, we can prove a relevant result that concerns the
second issue, i.e., the question regarding \textit{Interesting Turing
  Machines}, i.e., programs capable of generating interesting new
geometric results.




\section{Undecidability Result}
\label{sec:undecidabilityResult}

In section~\ref{sec:PuzzisClassification} the application of filters
was discussed, these filters are based on some measures of
interestingness that are still to be validated and that are applied in
an heuristic way. Is it possible to have a deterministic approach,
i.e., is it possible to write a computer program that in a
deterministic way, find interesting theorems? We show, as an
application of the Rice's
theorem~\cite{rice1953,rogers1987theory,Sipser1997} (see
Lemma~\ref{the:RiceTheorem}), that it is undecidable to determine,
for a given Turing Machine, whether the language recognised by it has
the (non-trivial) property of finding interesting theorems.

\begin{definition}[Non-Trivial Property]
  \label{def:NonTrivialProperty}
  A property $p$ of a formal language is \emph{non-trivial} if:
  \begin{itemize}
  \item there exist a recursively enumerable language having the
    property $p$;
  \item there exist a recursively enumerable language not having the
    property $p$.
  \end{itemize}
\end{definition}

\begin{lemma}[Rice's Theorem]
  \label{the:RiceTheorem}
  Let $p$ be any non-trivial property of the language of a Turing
  machine. The problem of determining whether a given Turing machine's
  language has property $p$ is undecidable.
\end{lemma}

\begin{theorem}[Undecidability Result]
  \label{the:undecibilityresult}  
  For any given Turing Machine, it is undecidable to determine,
  whether the language recognised by it has the property of finding
  interesting theorems.
\end{theorem}

\begin{proof}
  All programs (Turing machines) capable of automated theorem proving
  and by extension generating/finding geometric theorems rely on a
  formal language to describe the geometric constructions, conjectures
  and proofs. For example we can consider the (full) \emph{First-Order Form}
  (FOF)\footnote{\url{http://tptp.cs.miami.edu/TPTP/QuickGuide/}} of
  \emph{TPTP}~\cite{Sutcliffe2017} and the formal axiomatic theories for
  geometry based on that language.\footnote{TPTP Axioms Files for
    geometry,
    \url{https://www.tptp.org/cgi-bin/SeeTPTP?Category=Axioms}, e.g.
    \emph{Tarski geometry axiom}, \texttt{GEO001} and \texttt{GEO002},
    \emph{Deductive Databases Method in Geometry}, \texttt{GEO012}.}


  Let $p$ be the property of that language that says that theorem $t$
  is interesting, for any conceivable definition of interestingness,
  then there exist a recursively enumerable language having the
  property $p$. It will be enough to restrict the language in such a way
  that the theorem $t$, and only this, would be recognised.  But,
  it also exist a recursively enumerable language not having the
  property $p$. It would be enough to restrict that language in such a
  way that only tautologies would be recognised. Tautologies are, for
  any conceivable definition of interestingness, uninteresting.  We
  have proved that $p$, the property that can establish if a given
  theorem is interesting, is a non-trivial property.
  
  Having establish that the property $p$ is non-trivial, then, by
  application of Rice's theorem, it is undecidable to determine for
  any given Turing machine $M$, whether the language recognised by $M$
  has the property $p$.  
\end{proof}

In other words, it is undecidable to have a deterministic program that
can find interesting problems. At best this is a task to be addressed
by programs based on algorithms guided by heuristics criteria.

\section{Designing Interesting Surveys}
\label{sec:surveys}

In light of our undecidability result, to understand what experts mean
by, ``a program that is able to also prove interesting theorems'',
must be done referring to empirical data, via the formulation of an
expert survey. However, for it to be fulfilled, one has to first reach
a minimal degree of agreement on the definition of interestingness of
theorems. How could one speak about programs that produce such
theorems? In order to achieve this agreement, an empirical exploration
of the notion of interestingness and of what it concretely entails is
paramount. This exploration requires to situate the notion of
interestingness historically and socio-culturally, considering
logical, epistemological, sociological, cognitive, semiotic and
pedagogical aspects of the issue.  Probably---and as Wos already
implies---interestingness entails different tangible properties, which
differ in given centuries, geographical locations and
societies. Moreover, in some cases we say that a theorem is
interesting for what we can call \textit{global reasons} e.g.,
Euclid's theorem on the infinitude of the set of prime numbers, Zorn's
lemma and G\"odel's Theorems are interesting due to their role in
mathematics, logic and computer science. Other times for \textit{local
  reasons} e.g., in relation to what we are teaching our students at
that moment. In order to assess which tangible properties---both
global and local---interestingness entails today, we are proposing to
conduct two \textit{expert surveys} with two statistically significant
pools of participants.

\paragraph{Influencing factors.} Gao et al. performed an
extensive analysis of areas like Set Theory, Peano's arithmetic and
Tarski's Geometry, looking for the relevance of structural factors,
such as the degree of logical connectives in the theorem, the
propositional schema of the formula formalising the theorem, the
abstract level of predicates and functions in the theorem and the
deduction distance of a
theorem~\cite{Colton2000,Gao2015a,Gao2019,Puzis2006}. Some of
these structural aspects might be related to our cognitive
dynamics. But also the epistemological role of a theorem with respect
to other theorems might be a relevant feature; or the educational role
that some theorems have with respect to some notions might influence
their interestingness. Finally, the history of a
theorem---e.g. Fermat's last theorem---could add points to its
interestingness, which, in the case of Fermat's last theorem, might be
already caused by the technicalities of the proof itself.

\paragraph{Designing the surveys.} Taking all these factors into
consideration, we would propose to design three surveys that question
experts from different fields.

Before describing the surveys below some clarifications are necessary.
We will use the term ``expert'' to mean mathematics teachers at
primary, middle, and high schools, and professors or researchers in
pure and applied mathematics at universities or at research centres.
Furthermore, we will focus on the case study of geometry, hence
interesting theorems in geometry. The reasons for this restriction to
geometry are as follows: on the one hand, considering all fields of
research in mathematics might require a too large number of experts
and could produce too many divergent ideas. On the other hand, having
in mind an application of the results in automatic theorem proving as
a target, it seems appropriate to move into an area were there are
many different methods and many automated provers implementing those
methods. Finally, geometry is a kind of language common to many areas
of mathematics and has been a domain for reflection since the early
years of mathematics teaching.

Finally, these surveys are intended to involve mathematics teachers,
but their outcome does not target mathematics education. Of course,
this is a possible target, but it is not the primary goal of these
starting surveys.

\subsection{Three Surveys}
In the first survey, we will ask the experts both to indicate some
situations in which they remember to have used the adjective
interesting concerning a theorem, and to explain the use of this
expression. In addition, we will ask experts to list several geometric
theorems they find interesting, and to list several geometric theorems
they find not interesting, both from elementary and higher geometry,
explaining the reasons for their answers (see
Appendix~\ref{sec:FirstSurvey}). This first survey is already under
way, the steering committee is already approaching it and the authors
of accepted papers in the conference, \emph{14th International
  Conference on Automated Deduction in Geometry} (ADG 2023)\footnote{
  ADG 2023, 14th International Conference on Automated Deduction in
  Geometry, Belgrade, Serbia, September 20-22, 2023.}, were invited to
participate. We are planning to enlarge it to our network of contacts
and we invite the interested reader to also participate, answering
it.\footnote{\url{https://docs.google.com/forms/d/e/1FAIpQLScIXZbLPBHTLvmQ28P30Cm_-lkOrM7e6rab7ho0WrAFwf_mbQ/viewform?usp=sf_link}}
We are planning to begin collect and analyse the answers in February,
2024.

We will use the information from this survey to define a list of
characteristics (A,B,C, \dots) of a theorem that offer sufficient
reasons to attribute interestingness to it.  We will assign weights to
the various characteristics by considering the answers to this first
survey.

After the first survey, we will implement a second one. This second
survey will consider a list of theorems that, in different
percentages, have the characteristics inferred from the first survey.
We will submit the second survey to a set of experts different from
those used in the first survey. We will ask these experts whether they
find the theorems listed interesting or not. We will ask them to rate,
using a Likert scale,\footnote{A Likert scale is a question which is a
  five-point or seven-point scale. The choices range from Strongly
  Agree to Strongly Disagree so the survey maker can get a holistic
  view of people's opinions. It was developed in 1932 by the social
  psychologist Rensis Likert.}  the degree of impact that having
certain characteristics plays in their attribution of interestingness
(see Appendix~\ref{sec:SecondSurvey}).

This second group will allow us to understand whether the
characteristics isolated through the first survey are sufficient
conditions to affirm that a theorem is interesting.

With an agreement on what an interesting theorem is, based on
empirical research, we could query experts in theorem
generators/finders design, with another survey (the third survey)
asking how to design software able to produce these interesting
theorems.

\medskip

After that, we will focus our empirical inquiry on programs, driven by
heuristics based on our findings, able to find interesting theorems.

We have established a \emph{steering committee} to design the surveys
and who will oversee the submission of the surveys to experts around
the world.

The steering committee consists of the following scholars:

\begin{itemize}
\item Thierry Dana-Picard, Jerusalem College of Technology, Jerusalem, Israel;
\item James Davenport, University of Bath, United Kingdom;
\item Pierluigi Graziani, University of Urbino, Urbino, Italy;
\item Pedro Quaresma, University of Coimbra, Coimbra, Portugal;
\item Tom{\'a}s Recio, University  Antonio de Nebrija, Madrid, Spain.
\end{itemize}

\section{Conclusions}
\label{sec:conclusions}

The pursuit of new and interesting theorems in geometry, by automatic
means is an interesting open problem. From the point of view of
generating new information the deductive approach seems the most
appropriated, given that: only logical consequences are ever generated
and also the paths to those new theorems can be analysed from the
point of view of the geometric theory used, i.e. in the process of
generating new facts, geometric proofs of their validity are
produced. Already existing implementations, e.g. the deductive
databases method (DDM) implemented in \emph{JGEx}, and new
implementations, e.g. the \emph{GeoGebra Discovery} and the new
implementation of the DDM, the \emph{OGPCP-GDDM} prover, can be
used. The separation of the uninteresting, trivial facts, from the new
and interesting facts is much harder. The current approaches are based
in ad-hoc measures, proposed by experts from the field, but
nevertheless, not substantiated by any study approaching that
problem. Our goal is to fulfil that gap, to produce a comprehensive
survey, supported in a large set of mathematicians, in order to be
able to return to that question and to develop filters supported by
the findings of that survey.

\paragraph{Acknowledgements} The authors wish to thank Francisco Botana,
Thierry Dana-Picard, James Davenport and Tom{\'a}s Recio for their
support in the pursue of this long term project.


\newcommand{\noopsort}[1]{}\newcommand{\singleletter}[1]{#1}


\appendix
\section{First Survey---Interesting Theorems}
\label{sec:FirstSurvey}

With this survey the goal will be to find the characteristics that make
a theorem interesting, or not. A list of questions about geometric
theorems found to be interesting, or not interesting.

\medskip

For an initial pool of expert on the area it is our intention to use
the network created for the submission of the COST proposal,
\emph{iGEOMXXI}.\footnote{OC-2020-1-24509, Building a Networked
  Environment for Geometric Reasoning (iGEOMXXI), The submitted Action
  (not funded) focused on the exploration of new paradigms and
  methodologies for supporting formal reasoning in the field of
  Geometry. A network of 49 experts from 19 countries.} This survey
will be available online, based on an online survey
tool.\footnote{e.g. \emph{LimeSurvey},
  \url{https://www.limesurvey.org/}}

\subsection{Interesting and Why?}
\label{sec:InterestingandWhy}

A list of situations/explanations about interesting theorems.

\medskip

\begin{center}
  \framebox{\parbox{0.975\textwidth}{Can you describe in detail a
    situation (during classes or lectures) in which you have used the
    adjective interesting applied to a theorem in geometry?

\begin{description}
\item [$n$th Situation] \hrulefill 
\par
\hrulefill 
\par
\hrulefill 
\par
\hrulefill 
\end{description}

Can you explain in detail the reasons why you used the adjective
interesting in the first situation?

\begin{description}
\item [$n$th Explanation] \hrulefill 
\par
\hrulefill 
\par
\hrulefill 
\par
\end{description}
}}
\end{center}

\subsection{Five Interesting Theorems in Geometry}
\label{sec:FiveInterestingTheoremsinGeometry}

A list of 5 questions, each about an interesting theorem.

\bigskip

\begin{center}
  \framebox{\parbox{0.975\textwidth}{Can you list at least five theorems in
    geometry that you consider interesting?

\begin{description}
\item [Theorem $n$] \hrulefill 
\par
\hrulefill 
\par
\hrulefill 
\par
\end{description}

Can you explain in detail the reason for your choice by listing at
least five adjectives that describe characteristics of the previous
theorem making it interesting?

\begin{itemize}
\item[] \hrulefill 
\par
\hrulefill 
\par
\hrulefill 
\par
\end{itemize}
}}
\end{center}

\subsection{Five Not Interesting Theorems in Geometry}
\label{sec:FiveNotInterestingTheoremsinGeometry}

A list of 5 questions, each about a not interesting theorem.

\bigskip

\begin{center}
  \framebox{\parbox{0.975\textwidth}{Can you list at least five theorems in
    geometry that you consider NOT interesting?

\begin{description}
\item [Theorem $n$] \hrulefill 
\par
\hrulefill 
\par
\hrulefill 
\par
\end{description}

Can you explain in detail the reason for your choice by listing at
least five adjectives that describe characteristics of the previous
theorem making it NOT interesting?

\begin{itemize}
\item[] \hrulefill 
\par
\hrulefill 
\par
\hrulefill 
\par
\end{itemize}
}}
\end{center}

\section{Second Survey---Characteristics of Interesting Theorems}
\label{sec:SecondSurvey}

This survey will only be designed after studying the results of the
first survey. The second survey will propose theorems (taken from the
first survey) and will provide characteristics (taken from the first
survey) for each of them. The survey will ask the participants to
express their opinion on characteristics that (presumably) make the
theorems interesting or not interesting.

\medskip

This survey will be available online, based on an online survey
tool.\footnotemark[9]

\medskip

\begin{center}
  \framebox{\parbox{0.975\textwidth}{Please express whether you consider the
    following theorems interesting or not, and why?
    
    \begin{description}
    \item Is Theorem $n$ interesting?\\
      $\square$ YES \quad $\square$ NO
      
      \bigskip
      Why? Because it has the characteristic A.\\
      $\square$ Strongly disagree {\ } $\square$ Disagree {\ }  
      $\square$ Neutral  {\ } 
      $\square$ Agree  {\ } 
      $\square$ Strongly Agree
      
      \medskip
      Why? Because it has the characteristic B.
      
      $\square$ Strongly disagree {\ }
      $\square$ Disagree{\ }
      $\square$ Neutral {\ }
      $\square$ Agree{\ }
      $\square$ Strongly Agree{\ }
      
      \medskip
      Why? Because it has the characteristic C.
      
      $\square$ Strongly disagree {\ }
      $\square$ Disagree{\ }
      $\square$ Neutral {\ }
      $\square$ Agree{\ }
      $\square$ Strongly Agree
      
      \medskip
      Why? Because it has the characteristic D.
      
      $\square$ Strongly disagree {\ }
      $\square$ Disagree{\ }
      $\square$ Neutral {\ }
      $\square$ Agree{\ }
      $\square$ Strongly Agree\\
    \end{description}
  }}
\end{center}


\end{document}